\documentclass{article}

\usepackage{geometry}
\newgeometry{vmargin={35mm}, hmargin={25mm,30mm}} 

\usepackage{cite}

\usepackage[hyphens]{url}
\usepackage{hyperref}
\usepackage{breakurl}

\usepackage{amsmath,amssymb,amsthm,amsfonts}
\usepackage{algorithmic}
\usepackage{graphicx}
\usepackage{textcomp}
\usepackage{color}

\usepackage{bm}
\makeatletter
\AtBeginDocument{\DeclareMathVersion{bold}
\SetSymbolFont{operators}{bold}{T1}{times}{b}{n}
\SetMathAlphabet{\mathrm}{bold}{T1}{times}{b}{n}
\SetMathAlphabet{\mathit}{bold}{T1}{times}{b}{it}
\SetMathAlphabet{\mathbf}{bold}{T1}{times}{b}{n}
\SetMathAlphabet{\mathtt}{bold}{OT1}{pcr}{b}{n}
\SetSymbolFont{symbols}{bold}{OMS}{cmsy}{b}{n}
\renewcommand\boldmath{\@nomath\boldmath\mathversion{bold}}}
\makeatother

\def\BibTeX{{\rm B\kern-.05em{\sc i\kern-.025em b}\kern-.08em
    T\kern-.1667em\lower.7ex\hbox{E}\kern-.125emX}}

\usepackage{mathtools}
\usepackage{booktabs}

\usepackage{siunitx}

\usepackage{authblk}
\usepackage{mathptmx}

\newtheorem{theorem}{Theorem}[section]

\newtheorem{remark}[theorem]{Remark}

\newcommand{\norm}[1]{\lVert#1\rVert}

\usepackage{graphicx} 
\DeclareMathOperator*{\argmin}{arg\,min}
\graphicspath{ {./images/} }

\newcommand{\ew}{\mathbb{E}}
\newcommand{\loss}{\mathcal{L}}
\newcommand{\R}{\mathbb{R}}
\newcommand{\J}{\mathcal{J}}
\newcommand{\id}{\operatorname{Id}}

\newcommand{\ny}{\texttt{NN2I[y]}}
\newcommand{\nz}{\texttt{NN2I[z]}}
\newcommand{\nys}{\texttt{NN2Is[y]}}
\newcommand{\nzs}{\texttt{NN2Is[z]}}
\newcommand{\nyt}{\texttt{NN2N[y]}}
\newcommand{\nzt}{\texttt{NN2N[z]}}
\newcommand{\nii}{\texttt{N2I}}

\newcommand{\noisea}{\xi}
\newcommand{\noiseb}{\eta}

\newcommand{\std}{\delta}
\newcommand{\cl}{\sigma}

\hypersetup{
breaklinks=true,   
colorlinks=true,   
pdfusetitle=true,  
}

\allowdisplaybreaks
\begin{document}

\title{Noisier2Inverse: Self-Supervised Learning for Image Reconstruction with Correlated Noise}

\author[1,2]{Nadja Gruber}
\author[3]{Johannes Schwab}
\author[3]{Markus Haltmeier}
\author[4]{Ander Biguri}
\author[1,2]{Clemens Dlaska}
\author[5]{Gyeongha Hwang}

\affil[1]{Digital Cardiology Lab, Medical University of Innsbruck, Austria}
\affil[2]{University Clinic of Internal Medicine III, Cardiology and Angiology, Medical University of Innsbruck, Austria}
\affil[3]{Department of Mathematics, University of Innsbruck, Austria}
\affil[4]{Department of Applied Mathematics and Theoretical Physics (DAMTP), University of Cambridge, UK}
\affil[5]{Department of Mathematics, Yeungnam University, Gyeongsan, Gyeongbuk 38541, South Korea}





\maketitle

\begin{abstract}
We propose Noisier2Inverse, a correction-free self-supervised deep learning approach for general inverse problems. The proposed method learns a reconstruction function without the need for ground truth samples and is applicable in cases where measurement noise is statistically correlated. This includes computed tomography, where detector imperfections or photon scattering create correlated noise patterns, as well as microscopy and seismic imaging, where physical interactions during measurement introduce dependencies in the noise structure. Similar to Noisier2Noise, a key step in our approach is the generation of noisier data from which the reconstruction network learns. However, unlike Noisier2Noise, the proposed loss function operates in measurement space and is trained to recover an extrapolated image instead of the original noisy one. This eliminates the need for an extrapolation step during inference, which would otherwise suffer from ill-posedness. We numerically demonstrate that our method clearly outperforms previous self-supervised approaches that account for correlated noise.
\end{abstract}


\section{Introduction}

Image reconstruction is central for various imaging modalities, including X-ray Tomography, Photoacoustic Tomography, Magnetic Resonance Imaging (MRI), Seismic Imaging, and Cryo-Electron Tomography. In many cases, the relationship between the ideal measured data and the reconstructed image can be modeled as a linear operator, resulting in an inverse problem of the form
\begin{equation} \label{eq:ip}
	y = A x + \noisea \,,
\end{equation}
where $x \in \mathbb{R}^{n}$ is the desired unknown, $\noisea\in\R^m$ is the noise, $y \in \mathbb{R}^m$ is the measured data, and $A \colon \mathbb{R}^n \to \mathbb{R}^m$ is the forward map. Due to ill-conditioning of $A$, unregularized reconstruction methods are unstable with respect to noise \cite{38288208ad764ab28d3b086f7a52e54f,scherzer2009variational,hansen2010discrete}, leading to significant errors in the reconstructed image. In addition, they suffer from incomplete data, making accurate image formation even more challenging. Therefore, finding an effective and accurate reconstruction function $B \colon \mathbb{R}^m \to \mathbb{R}^n$ is both relevant and challenging.

In many real-world inverse problems, various optical, physical, and procedural factors introduce spatially correlated noise. For example, in computed tomography (CT), several factors contribute to spatial correlation \cite{barrett2004artifacts,diwakar2018review}. Metallic implants lead to beam hardening and scattering, introducing structured noise with a complex distribution \cite{whiting2006properties}, often with unknown parameters. Non-uniform detector sensitivities or defects result in fixed-pattern noise, while beam hardening and X-ray scattering contribute to spatially varying noise, especially under low-dose conditions. Furthermore, reconstruction artifacts arise from backprojection, which spreads noise from individual detector pixels across lines in the reconstructed image, creating statistical dependencies among pixels. Similarly, in microscopic imaging, spatially correlated noise arises from various sources. Optical aberrations caused by lens imperfections distort light paths, introducing noise patterns. Sample-dependent factors, such as refractive index variations, autofluorescence, or sample thickness, can mimic noise. Detector non-uniformities result in fixed-pattern noise, while environmental and mechanical factors, including vibrations, sample drift, and laser instabilities, contribute to motion blur and uneven illumination.

\subsection{Learned image reconstruction}

Learning-based reconstruction methods achieve state-of-the-art performance in a wide range of inverse imaging problems, from medical imaging to computational photography \cite{ongie2020deep,arridge2019solving,haltmeier2023regularization}. However, these approaches typically rely on training data  $\{ (x_t, y_t) \colon  t =1, \dots, T \} $ consisting of paired clear ground-truth images $x_t$ and noisy measurements $y_t$, allowing the reconstruction function to calibrate such that $B(y_t) \simeq x_t$.  
However, collecting ground-truth images can be costly or even unattainable, especially in medical and scientific imaging. To address this challenge, self-supervised learning techniques have emerged as a promising alternative, using only noisy and potentially incomplete measurement data $\{ y_t \colon  t =1, \dots, T \} $ for training \cite{ millard2024clean, blumenthal2024self, krull2019noise2void,hendriksen2020noise2inverse,tachella2024unsure}. This circumvents the need for fully sampled training datasets with corresponding clean targets, making it applicable to scenarios where clean ground-truth data are difficult or impossible to obtain. Unlike supervised learning, which uses information in the data pairings, self-supervised learning exploits inherent structures or relationships within the measurement data for training.

Several self-supervised imaging methods have been proposed for specific scenarios. Noise2Noise~\cite{lehtinen2018noise2noise} trains models using pairs of noisy realizations of the same signal. Although effective, multiple realizations of the same object are often unavailable in practice. Noise2Self \cite{batson2019noise2self} and Noise2Void~\cite{krull2019noise2void} overcome this limitation by only requiring noisy observations of the same object. However, they assume that noise in one subgroup of pixels is statistically independent of noise in another. Noise2Inverse~\cite{hendriksen2020noise2inverse}, an extension of Noise2Self designed for inverse problems, also assumes uncorrelated noise (in the data domain).   Noisier2Noise \cite{yaman2020self} allows correlated noise but requires a two-step approach for general inverse problems, where image formation and denoising are separated. These methods also face the challenge of a fixed null-space component due to the ill-posedness of the inverse problem.

\subsection{Main contributions}

To address the limitations of existing self-supervised techniques, we propose Noisier2Inverse, a one-step method for solving inverse problems with correlated noise. Given noisy data $y_t$, we create noisier data $y_t + \eta_t$ and perform inference $(f_\theta^X \circ B^\sharp) ( y_t + \eta_t)$ using a fixed  initial reconstruction map $B^\sharp \colon \R^m \to \R^n$ and a neural network $f_\theta^X \colon \R^n \to \R^n$ near minimizing the loss function
\begin{equation} \label{eq:nn2iD}
\loss(\theta) \triangleq \sum_{t=1}^T \| A [ (f_\theta^X \circ B^\sharp) ( y_t + \eta_t)  ] -  (y_t - \eta_t) \|_2^2 \,.
\end{equation}
It shares similarities with the Noisier2Noise framework in the sense that both use noisier images $y_t + \eta_t$ for the network training. Noisier2Noise, however, is designed for the case $A= \id$, minimizes $\sum_{t=1}^T \|  f_\theta ( y_t + \eta_t)  ] - y_t   \|_2^2$  and inference is done by $
2 f_\theta (y_t + \eta_t) - (y_t + \eta_t)$. 
Note that, in practice, $\noiseb_t$ is resampled at each training iteration.
Noisier2Inverse differs in two key aspects, significantly improving  performance: 

\begin{itemize}
\item \textbf{One-step approach:}  
When using Noisier2Noise directly in the context of image reconstruction, one has to follow a two-step procedure. This can be done either in a preprocessing manner, $B = B^\sharp \circ f^Y_\theta$, with $f^Y_\theta$ minimizing $\sum_{t=1}^T \| f^Y_\theta (y_t + \eta_t) - y_t \|_2^2$, or in a postprocessing manner, $B = f^X_\theta \circ B^\sharp$, with $f^X_\theta$ minimizing $\sum_{t=1}^T \| f^X_\theta (B^\sharp y_t + B^\sharp \eta_t) - B^\sharp y_t \|_2^2$. In both modes, it is impossible to restore image features outside the range of $B^\sharp$, resulting in severe reconstruction artifacts. Noisier2Inverse overcomes these limitations by using a one-step approach, targeting $f^X_\theta \circ B^\sharp$ in image space but with a loss in the measurement domain.

\item \textbf{Extrapolation-free:}  
Noisier2Noise generalized to a one-step method  for image reconstruction minimizes  
$\sum_{t=1}^T \| A [ f^X_\theta \circ B^\sharp (y_t + \eta_t) ] - y_t \|_2^2$  
in combination with an extrapolation step included in the inference 
$2 f^X_\theta \circ B^\sharp (y + \eta) - B^\sharp (y + \eta)$ (see Remark~\ref{rem:n2n-onestep}).  
However, subtracting the initial reconstruction $B^\sharp (y + \eta)$ requires a trade-off between the stability and accuracy of the initial reconstruction.   In contrast, Noisier2Inverse does not require such an extrapolation and thus is free from this hurdle. The key idea is to target $y_t - \eta_t$ instead of $y_t$ via the network. While targeting $y_t$ (given noisy data) with input $y_t + \eta_t$ (noisier data), Noisier2Noise only performs half of the desired denoising, as it fails to distinguish between the original noise $\noisea_t$ and the added noise $\noiseb_t$ requires the extrapolation step of adding $f^X_\theta \circ B^\sharp (y + \eta) - B^\sharp (y + \eta)$ to $f^X_\theta \circ B^\sharp (y + \eta)$. In contrast, Noisier2Inverse targets $y_t - \eta_t$, and although it again cannot distinguish between the original and added noise, targeting $y_t - \eta_t$ will naturally yield noise-free data.         
\end{itemize}

We emphasize that we do not suggest strictly minimizing the loss in \eqref{eq:nn2iD}, as this may result in severe overfitting and the amplification of noise \cite{scherzer2009variational}. Thus, we propose including additional regularization in the form of early stopping when using minimization algorithms for \eqref{eq:nn2iD}, as will be described in Remark \ref{rem:stopping}.

Our presented experimental results for CT image reconstruction clearly show that Noisier2Inverse outperforms both Noisier2Noise (even after adjustment to one-step image reconstruction) and Noise2Inverse, indicating that it fills a gap in existing self-supervised image reconstruction with correlated noise.

\section{Theory}

Throughout this paper, we work in a Bayesian context where $X$, $Y$, and $\Xi$  are random variables with realizations $x$, $y$, and $\noisea$  subject to  \eqref{eq:ip}. Our goal is to define a reconstruction function $B \colon \mathbb{R}^m \to \mathbb{R}^n$ that recovers $x$ from $y$, evaluating the reconstruction based on the $\ell^2$ norm. In this context, the ideal reconstruction function minimizes the supervised risk $\ew [\| B(Y) - X \|^2 ]$. 

The challenge is that the distributions of $X$ and $Y$ are unknown; instead, only samples of $
(X, Y)$ in the supervised case and samples of $Y$ in the self-supervised case are available. The key idea of self-supervised image reconstruction is to find and minimize a surrogate risk that does not depend on $X$, while leveraging additional knowledge about $A$, the optimal reconstruction function $B$, and the underlying statistical noise model.

\subsection{Prior work}
\label{sec:back}

Before presenting the proposed Noisier2Inverse in Subsection~\ref{sec:npi}, we first describe the most relevant works that provide the background for Noisier2Inverse. In fact, Noisier2Inverse can be seen as an extrapolation-free, one-step extension of Noisier2Noise for image reconstruction.

\subsubsection{Self-supervised image denoising}

Image denoising aims to solve \eqref{eq:ip} when $A$ is the identity matrix. In this case, we aim for a denoising function, $f \colon \mathbb{R}^n \to \mathbb{R}^n$ mapping noisy images $y \in \R^n$ to clean images $x  \in \R^n$.  

\paragraph{Noise2Self:}  
Consider a partition $\J$ of the pixel set $\{ 1, 2, \dots, n\}$ and write  $x_I = (x_i)_{i\in I}$ and $x_I^c = (x_i)_{i \not\in I}$. The basic idea of Noise2Self is to restrict the class of denoisers to $\J$-invariant functions $f \colon \R^n \to \R^n$ defined by the property that $(f(x))_I$ depends only on $x_I^c$ for any $x \in \R^n$ and any $I \in \J$. The seminal work \cite{batson2019noise2self} shows that $\ew [\| f(Y) - X \|^2 ] = \ew [\| f(Y) - Y \|^2 ] + \ew \| Y - X \|^2$ for any $\J$-invariant function $f$, provided that the noise $\noisea$ has zero mean and that $\Xi_I$ is independent of $\Xi_I^c$ for any $I \in \J$. In particular, $\ew \| f(Y) - X \|^2$ is minimized by the same $\J$-invariant function as the self-supervised risk $\ew \| f(Y) - Y \|^2$. The latter is nearly minimized by minimizing $\sum_{t =1}^T \| y_t - f(y_t) \|^2$ for the noisy samples $\{ y_t \colon t =1, \dots, T \}$. In \cite{batson2019noise2self}, it has further been proposed to construct $\J$-invariant functions as $(f(x))_I := (g(a_I + x_I^c))_I$, where $g \colon \R^n \to \R^n$ is taken as a standard network.

\paragraph{Noisier2Noise:}  
Unlike Noise2Self, Noisier2Noise \cite{moran2020noisier2noise} does not require the noise to be independent among certain subgroups of pixels in the image.  
The main idea is to further perturb $Y$ by adding noise from the same distribution, resulting in noisier data $Z$. The authors show that $\ew [X | Z] = 2 \ew [Y | Z] - Z$, which suggests training a denoiser that minimizes the surrogate loss $\ew [\| f(Z) - Y \|^2 ]$ and using $f(Z)$ followed by an extrapolation step of adding $f(Z) - Z$ for inference. In \cite{moran2020noisier2noise}, it is demonstrated that Noisier2Noise  performs comparably to learned methods requiring richer training data while also outperforming traditional non-learned denoising methods.

\subsubsection{Self-supervised image reconstruction}

Now we consider the general image reconstruction task \eqref{eq:ip} beyond denoising where $A$ is a linear, potentially ill-conditioned forward map.

\paragraph{Two-step approaches:}   \label{sec:twostep}
A natural strategy for extending self-supervised denoising to image reconstruction is to use a two-step approach, either in the preprocessing mode $B^Y = B^\sharp \circ f^Y$ or the post-processing mode $B^X = f^X \circ B^\sharp$, where the networks $f^X$ and $f^Y$ are trained as self-supervised denoisers using methods like Noise2Self or Noisier2Noise. Here, $B^\sharp$ denotes a non-learned initial reconstruction map. When combined with Noise2Self, the post-processing approach increases the correlation of the noise, while for Noisier2Noise, the extrapolation step tends to become unstable. In the context of subsampled MRI image reconstruction, SSDU \cite{yaman2020self} and its extensions~\cite{millard2023theoretical,millard2024clean,blumenthal2024self} successfully implement multiplicative Noisier2Noise in preprocessing mode, where $A$ is taken as the Fourier transform and subsampling is interpreted as multiplicative noise. Because the Fourier transform is an isometry, Noisier2Noise in the Fourier domain can be stably transferred to the image domain. However, this stability does not hold for general inverse problems, such as the Radon transform. Due to the absence of ground-truth images, addressing issues related to the null space and ill-posedness remains challenging.



\paragraph{Noise2Inverse and Sparse2Inverse:}  
Noise2Inverse~\cite{hendriksen2020noise2inverse} extends Noise2Self to image reconstruction in a one-step manner and is currently one of the best-performing self-supervised methods for CT reconstruction. Such an approach was proposed for Cryo-Electron Tomography in~\cite{buchholz2019cryo}, and a theoretical analysis in a general framework was presented in~\cite{hendriksen2020noise2inverse}. In this approach, a network of the form \( B^X = f^X \circ B^\sharp \) is trained in the reconstruction domain. However, $\J$-invariance is used in the  projection domain, where the noise is assumed to be statistically independent. While Noise2Inverse effectively reduces measurement noise, it is neither designed to address artifacts caused by incomplete data nor to handle correlated noise. This limitation is addressed in~\cite{gruber2024sparse2inverse} where the proposed Sparse2Inverse uses an architecture and data partitioning similar to Noise2Inverse but with a loss in the data domain. This allows the network to learn components of the null space, which has been demonstrated to successfully remove missing data ghosts for sparse-angle CT. A similar strategy was proposed in \cite{unal2024proj2proj} for low-dose CT.

\subsection{Noisier2Inverse}
\label{sec:npi}

In this section, we introduce Noisier2Inverse, a novel self-supervised framework for solving \eqref{eq:ip}, allowing general $A$ and  spatially correlated noise $\noisea$. The noise can be any type of additive noise with a known generation process.

\subsubsection{Theory}
\label{sec:theory}

As before,  $X$, $Y$, and $\Xi$ are random variables with realizations $x$, $y$, and $\noisea$, subject to model \eqref{eq:ip}. Further, let $N$ be another additive noise variable and consider the noisier data $Z = Y + N$. Similar to the Noisier2Noise framework, we aim for a loss function based on samples of $(Y,Z)$ instead of samples of $(X,Y)$. However, in contrast to Noisier2Noise,  our method avoids any extrapolation step, which would negatively affect  reconstruction quality due to the ill-conditioning of $A$.  

Our main theoretical result is as follows.

\begin{theorem}[Expected Prediction Error] \label{thm:theory}
Let  $(X,Y,Z)$  be as above. Then, for   any $W \in \R^{q \times m}$, we have    
\begin{equation}\label{eq:noisier2inv}
\begin{aligned}
& \argmin_f \ew \norm{WA [ f \circ B^\sharp (Z) ] -  WA [ X ] }_2^2 = \argmin_f \ew \|WA [ f \circ B^\sharp (Z) ]- W(2Y - Z) \|_2^2, 
\end{aligned}
\end{equation}
where $B^\sharp \colon \R^m \to \R^n$ is fixed and the  minimum is taken  over all measurable functions $f \colon \R^n \to \R^n$.
\end{theorem}

\begin{proof}
From the properties of the conditional expectation, it follows that
\begin{align*}
    &\ew \left[\|WA [ f \circ B^\sharp (Z) ] - W(2Y- Z)\|_2^2\right] 
     - \ew \left[\| WA [  f \circ B^\sharp (Z)  ] - WA [ X ] \|_2^2\right] 
    \\
   & \hspace{0.001\columnwidth} = \ew_Z \left[\ew \left[\|WA [ f \circ B^\sharp (Z) ] - W(2Y- Z) \|_2^2 \middle| Z\right] \right] 
     - \ew_Z \left[\ew \left[\|  WA [  f \circ B^\sharp (Z)  ] -  WA [X] \|_2^2 \middle| Z \right] \right] 
    \\
    & \hspace{0.001\columnwidth} = 2 \, \ew_Z \left[\ew 
    \left[(WA [  f \circ B^\sharp (Z)  ])^\top (W A X) \middle| Z \right]\right] 
     - 2 \, \ew_Z \left[\ew 
     \left[WA [  f \circ B^\sharp (Z)  ]^\top W(2Y - Z) \middle| Z \right]\right]  
     - \ew \left[\|WA X\|_2^2 \right] 
     +  \ew \left[\|W(2Y - Z)\|_2^2 \right] 
    \\
    & \hspace{0.001\columnwidth}
    = \ew_Z \left[2(WA [ f \circ B^\sharp (Z)])^\top W \ew \left[AX - (2Y - Z) \middle| Z\right]\right] 
     - \ew \left[\|WAX\|_2^2 \right] 
    + \ew \left[\|W(2Y - Z)\|_2^2 \right],
\end{align*}
where the first equality used the  law of total expectation.
Further, from $\ew[\Xi|Z] = \ew[N|Z]$, we get 
\begin{align*}
2 \, \ew[Y|Z] &= \ew[AX|Z] + \ew[AX|Z] + \ew[\Xi|Z] + \ew[N|Z]\\
&= \ew[AX|Z] + \ew[AX + \Xi + N|Z]\\
&= \ew[AX|Z] + \ew[Z|Z] \,.
\end{align*}
Thus,  $\ew[(AX)-(2Y - Z)|Z] = 0$, and  
\begin{align*}
& 
\ew \left[\|WA [ f \circ B^\sharp (Z) ] - W(2Y- Z)\|_2^2\right]
= \ew \left[\| WA [  f \circ B^\sharp (Z)  ] - WA [ X ] \|_2^2\right] 
- \ew \left[\|WAX\|_2^2 \right]  + \ew[\|W(2Y - Z)\|_2^2]
\end{align*}
which yields \eqref{eq:noisier2inv}.
\end{proof}

Based on Theorem~\ref{thm:theory}, we train $f_\theta \circ B^\sharp$ by nearly minimizing $\mathbb{E} \left[ \|W A [(f_\theta \circ B^\sharp)(Z)] - W (2Y - Z)\|_2^2 \right]$ and use $(f_\theta \circ B^\sharp)(Z)$ for inference. In contrast to Noisier2Noise, the trained network directly yields $X$ without requiring any extrapolation step, which we observed to result in significant improvements for problems such as CT reconstruction.

\subsubsection{Practical realization}
\label{sec:realization}

Following the self-supervised learning paradigm, we assume a collection of noisy measurement data $y_t = A x_t + \noisea_t \in \mathbb{R}^m$ for $t = 1, \dots, T$ with $y_t$, and $\xi_t$ being realizations of $X$, and $\Xi$ and without access to the clean ground truth realizations $x_t \in \mathbb{R}^n$ of $X$. Following Section~\ref{sec:theory}, in each training iteration we generate noisier data $z_t = y_t + \noiseb_t$ where $\noiseb_t$ are realizations of $N$ and approximately minimize the loss function  
\begin{equation}\label{eq:lossW}
\loss_{W}(\theta)  
= \sum_{t=1}^T 
\| W [ A(f_\theta \circ B^\sharp) (z_t) ] - W(2 y_t - z_t)\|_2^2 \,,
\end{equation}  
where $B^\sharp \colon \mathbb{R}^m \to \mathbb{R}^n$ is an initial reconstruction map, and $f_\theta \colon \mathbb{R}^n \to \mathbb{R}^n$ is a neural network. During inference, we define reconstructions by either of the following:  
\begin{align}
x_W^{(z)} &= (f_\theta \circ B^\sharp) ( z ) \label{eq:ni1}\\ 
x_W^{(y)} &= (f_\theta \circ B^\sharp) ( y ) \label{eq:ni2}
\end{align}  
where $z = y + \noisea$. While variant~\eqref{eq:ni1} is suggested by the theory, variant~\eqref{eq:ni2} yields better numerical results. Similar observations have been reported in \cite{millard2023theoretical, yaman2020self}.
  
\begin{remark}[Regularization via early stopping] \label{rem:stopping}
To ensure proper regularization, we do not strictly minimize the loss \eqref{eq:lossW}. Instead, we propose using standard gradient methods to minimize \eqref{eq:lossW} while incorporating early stopping in the iterative process. Our experiments revealed that training the networks for too long leads to a continued decrease in the loss but introduces unwanted noise into the network's output.

Determining an appropriate stopping criterion remains a significant challenge. To address this, we employed a very small learning rate and conducted extensive experiments, training the network for over 10,000 epochs. In our evaluation, we compared results obtained after training for a fixed number of epochs with those based on an optimal selection using PSNR (which, in the self-supervised setting, is not directly applicable due to the absence of clean data). Interestingly, we observed that the results did not differ significantly between these approaches.

\end{remark}

\begin{remark}[Sobolev  loss] 
The choice $W= \id$ results in the standard MSE-loss in measurement space. Due to the smoothing properties of the forward map $A$,  this loss might be too weak for accurate quantification of reconstruction quality; in particular, it does not strongly penalize noisy results. Thus, stronger norms for the loss seem reasonable. In particular, for the numerical results below for CT inversion we will also use   $W = \nabla$ as first order discrete derivative, in which case we refer to \eqref{eq:lossW} as the Sobolev loss.  
\end{remark}
 
\begin{remark}[One-step Noisier2Noise]   
Following \cite{moran2020noisier2noise} \label{rem:n2n-onestep} and similar to Theorem \ref{thm:theory} one derives     
\begin{equation}\label{eq:noisier2noise1}
\argmin_f \ew \norm{A [ f \circ B^\sharp (Z) ] -  A [ X ] }_2^2=
2 \argmin_f \ew \norm{A [ f \circ B^\sharp (Z) ]- Y }_2^2 - B^\sharp (Z)   \,.
\end{equation}
This suggests training a network $f_\theta \circ B^\sharp \colon \R^m \to \R^n$ by nearly minimizing $\sum_{t=1}^T \| A [ (f_\theta \circ B^\sharp )(z_t)] - y_t \|_2^2$ (again, including early stopping) and performing inference by  
\begin{align}
x^{(1)} &=  2 \, (f_\theta - \id ) ( B^\sharp( z ) ) \label{eq:nn1}\\ 
x^{(2)} &=  f_\theta ( B^\sharp ( y ) ) \,, \label{eq:nn2}
\end{align}  
where, as before, $B^\sharp$ is an initial  reconstruction, $f_\theta$ is an image space network, and $y$ and $z$ are the noisy and the noisier data.  Similarly to the proposed Noisier2Inverse, this can be seen as a one-step approach in the sense that initial image formation is not separated from self-supervised training. In our experiments, we will always refer to this as Noiser2Noise. Note that Formulas~\eqref{eq:noisier2noise1} and~\eqref{eq:nn1} reduce to the standard Noisier2Noise framework when  $A=\id$. However, to the best of our knowledge, no generalization for an arbitrary $A$ has been proposed in the literature.

Our experiments consistently demonstrate that Noisier2Inverse outperforms Noisier2Noise. We attribute this to imperfections in the reconstruction step, which introduce errors that propagate and accumulate during the subsequent extrapolation phase.
\end{remark}

\section{Numerical Experiments}

For the experimental results presented, we consider CT image reconstruction in a 2D parallel beam geometry, where 
$A$ represents the Radon transform with equidistant projection angles distributed along a semicircle. We consider both a full data situation and an incomplete data situation using only a small number of directions. The code together with the corresponding datasets can be found on GitHub (\url{https://github.com/Nadja1611/Noisier2Inverse-Joint-Denoising-and-Reconstruction-of-correlated-noise}).

\begin{figure*}[htb!]
    \centering
    \includegraphics[width=0.99\linewidth]{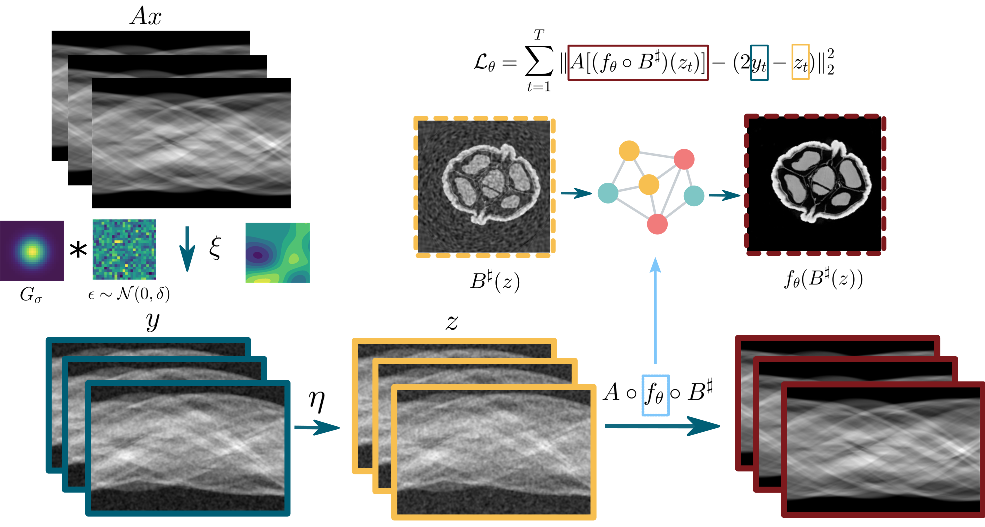}
    \caption{Noisier2Inverse for CT reconstruction: We first create noisy sinograms $y$ by adding correlated noise $\noisea$ (white noise convolved with a Gaussian), where $\cl$ is the correlation parameter of the noise and $\delta$ its standard deviation. At each training iteration, we then add additional noise to generate $z = y + \noiseb$  The initial reconstruction  $B^\sharp(z)$ serves as the network input, and subsequently, the loss is computed and minimized in the data domain by applying the forward operator $A$ to the network output.}
    \label{fig:workflow}
\end{figure*}

The overall workflow of  Noisier2Inverse for  CT reconstruction using data generation and image reconstruction  can be seen in Figure~\ref{fig:workflow}.

\subsection{Experimental details}

\subsubsection{Dataset generation}
We generate two datasets with correlated noise. The first one consists of 32 images from the CBCT Walnut dataset (Cone-Beam CT Data Set Designed for Machine Learning). The second dataset is a heart CT dataset consisting of 107 images from the Patient-Contributed Image Repository, taken from \url{https://www.kaggle.com/datasets/abbymorgan/heart-ct}. We split both datasets into training, validation, and test data and resized all images to a size of $336 \times 336$. We then construct the full projection data $A x$ using 512 projection angles. For $A$ and the initial reconstruction map $B^\sharp$ we use the Radon  transform and the FBP (filtered backprojection) of tomosipo~\cite{hendriksen-2021-tomos} and leveraged LION (\url{https://github.com/CambridgeCIA/LION/tree/main}) to specify the CT geometry.  

The clean data $A x$ are artificially distorted by adding noise twice, resulting in $y = A  x + \noisea$ and $z = y +  \noiseb$. 
Here, $\noisea$ and $\noiseb$ follow the same correlated nosie distribution  $G_\cl \ast \mathcal{N}(0,\std)$, where $G_\cl$ is a convolutional kernel with bandwidth $\cl$. We set $\std = 5.0$ for the Walnut dataset and $\std = 1.0$ for the heart CT dataset, with $\cl = 2.0, 3.0, 5.0$ in both cases.

\subsubsection{Tested methods}
{
The results use four variants of Noisier2Inverse, two variants of Noisier2Noise and Noise2Inverse, resulting in a total of seven comparison methods. In all Tables (see supplementary) and Figures we use the following notation:  
\setlength{\parskip}{0pt}
\begin{itemize}
\setlength{\itemsep}{0pt}

\item \nz{} for Noisier2Inverse with $W=\id$ on $z$.

\item \ny{} for Noisier2Inverse with $W=\id$ on $y$.

\item \nzs{} for Noisier2Inverse with $W=\nabla $ on $z$ 

\item \nys{} for Noisier2Inverse with $W=\nabla $ on $y$ 

\item \nzt{}  for Noisier2Noise evaluated on $z$  

\item \nyt{}  for Noisier2Noise evaluated on $y$  

\item \nii{} for Noise2Inverse.
\end{itemize}
}

\subsubsection{Evaluation metrics}

To analyze the training behavior, we validated the results and computed the peak signal-to-noise ratio (PSNR) and structural similarity index metric (SSIM) between the reconstructions and the clean images. We observed that, with an appropriately chosen learning rate, the method demonstrated excellent stability. For a sufficient number of training epochs, the results were robust and consistent.

\subsubsection{Network training}

We employ the same U-Net configuration as used in \cite{hendriksen2020noise2inverse} for all methods, set the learning rate to $5 \times 10^{-5}$ and utilize  a batch size of 4.  Notably, the results tend to degrade as the number of epochs increases. We hypothesize that this phenomenon is due to the correlation of noise present in the data. To mitigate this issue, we use PSNR validation on ground truth data as an oracle criterion. During training, the network was trained for 9000 epochs, and the weights from the final epoch were used for testing. These results were then evaluated against the PSNR-based oracle.

\subsection{Experimental results}

In the following experimental results we demonstrate  the  superiority of Noisier2Inverse over Noisier2Noise and Noise2Inverse in the presence of correlated noise. Moreover, we emphasize the benefits of using the Sobolev loss ($W= \nabla$) instead of the standard MSE-loss ($W= \id$) in certain cases, evaluate stopping criteria  and  demonstrate the robustness of Noisier2Inverse against variations in the noise model.

\begin{figure*}[htb!]
    \centering
\includegraphics[width=0.995\linewidth]{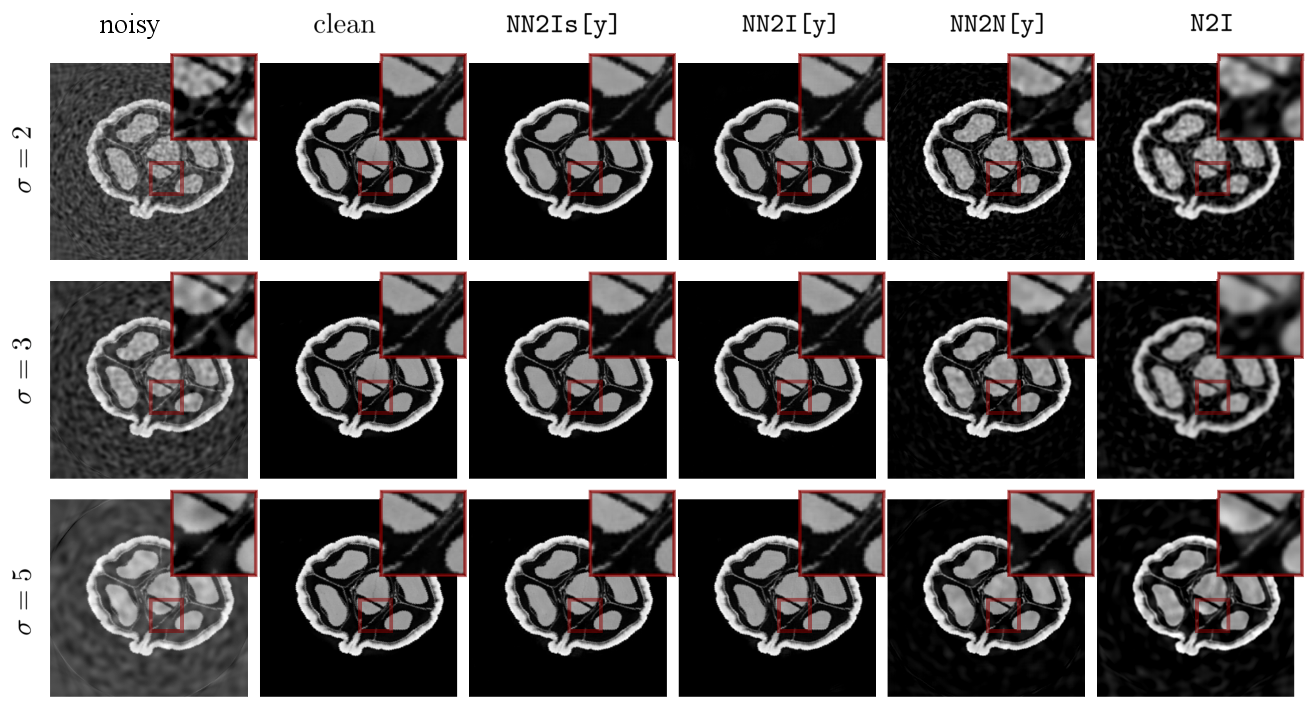} \\[1em]
\includegraphics[width=0.995\linewidth]{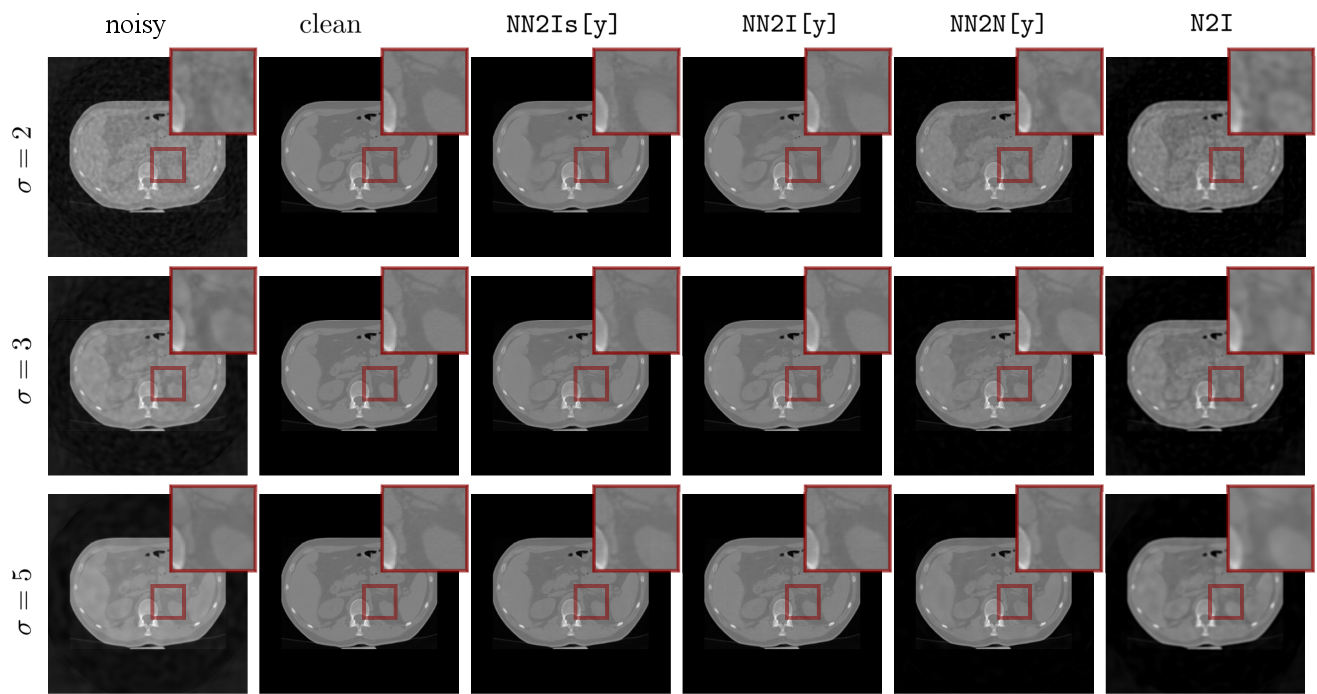}
    \caption{\textbf{Reconstruction results performing inference on the noisier data $y$} for the Walnut testing set (top) and the heart CT testing set (bottom).}
    \label{fig:results}
\end{figure*}

\subsubsection{Reconstruction results}

Figure~\ref{fig:results} shows the reconstruction results using different methods for the Walnut testing set (top) and the Heart CT testing set (bottom). We only display the variants where Noisier2Inverse and Noisier2Noise are applied to $y$ for inference; see the supplementary material for the variants with $z$. The results indicate a clear outperformance of Noisier2Inverse over Noisier2Noise and Noise2Inverse. Moreover, Noisier2Inverse with Sobolev loss (\nys) outperforms Noisier2Inverse with MSE-loss (\ny) in regions expected to remain constant. Conversely, \ny{} excels in preserving fine details. Thus, the choice between these methods should depend on the particular scenario.

\begin{figure*}[htb!]
    \centering
\includegraphics[width=0.99\linewidth]{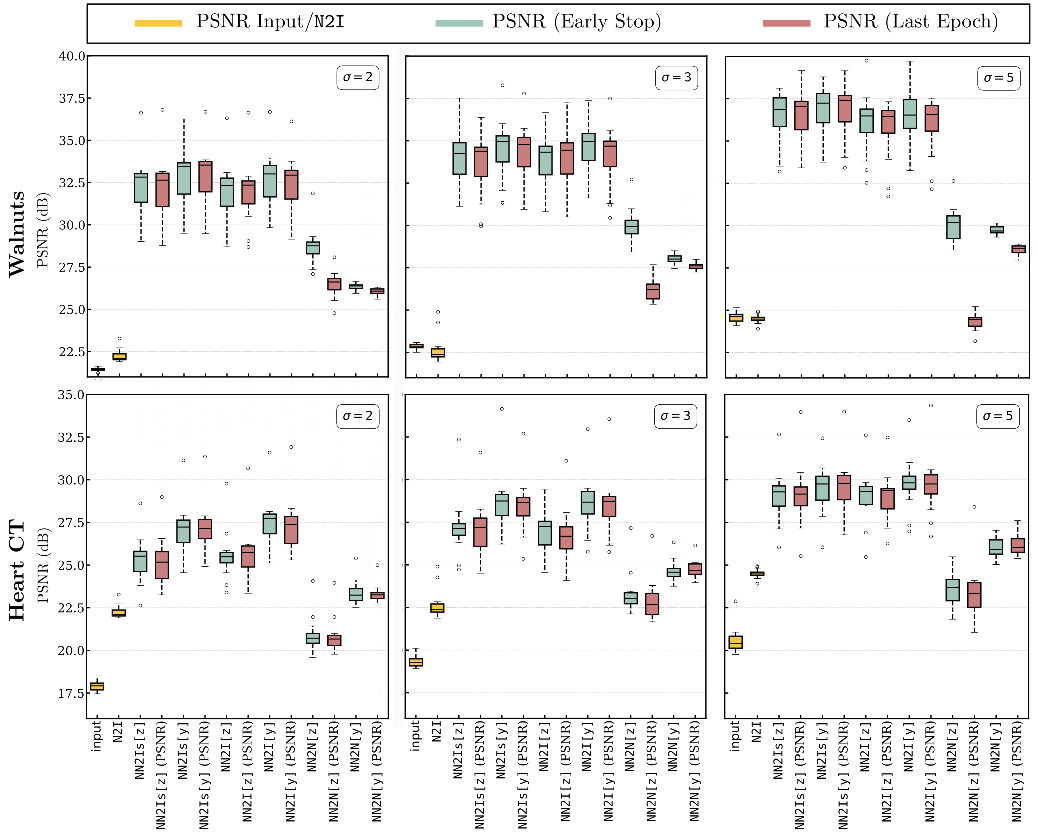}
    \caption{PSNR values for the different methods and for two different stopping criteria evaluated for correlation parameter $\cl$ of the noise.}
    \label{fig:boxplots} 
\end{figure*}

Figure~\ref{fig:boxplots} shows boxplots for the PSNR values for all seven different methods. The comparison includes performance when the model from the last training epoch is used for testing and when the best-performing model is selected based on the PSNR-based oracle criterion. We observe that the performance difference between unsupervised stopping and the oracle is minimal. For Noisier2Noise, this gap becomes slightly  more pronounced, highlighting another advantage of Noisier2Inverse. 
\begin{figure*}[htb!]
    \centering
\includegraphics[width=0.8\linewidth]{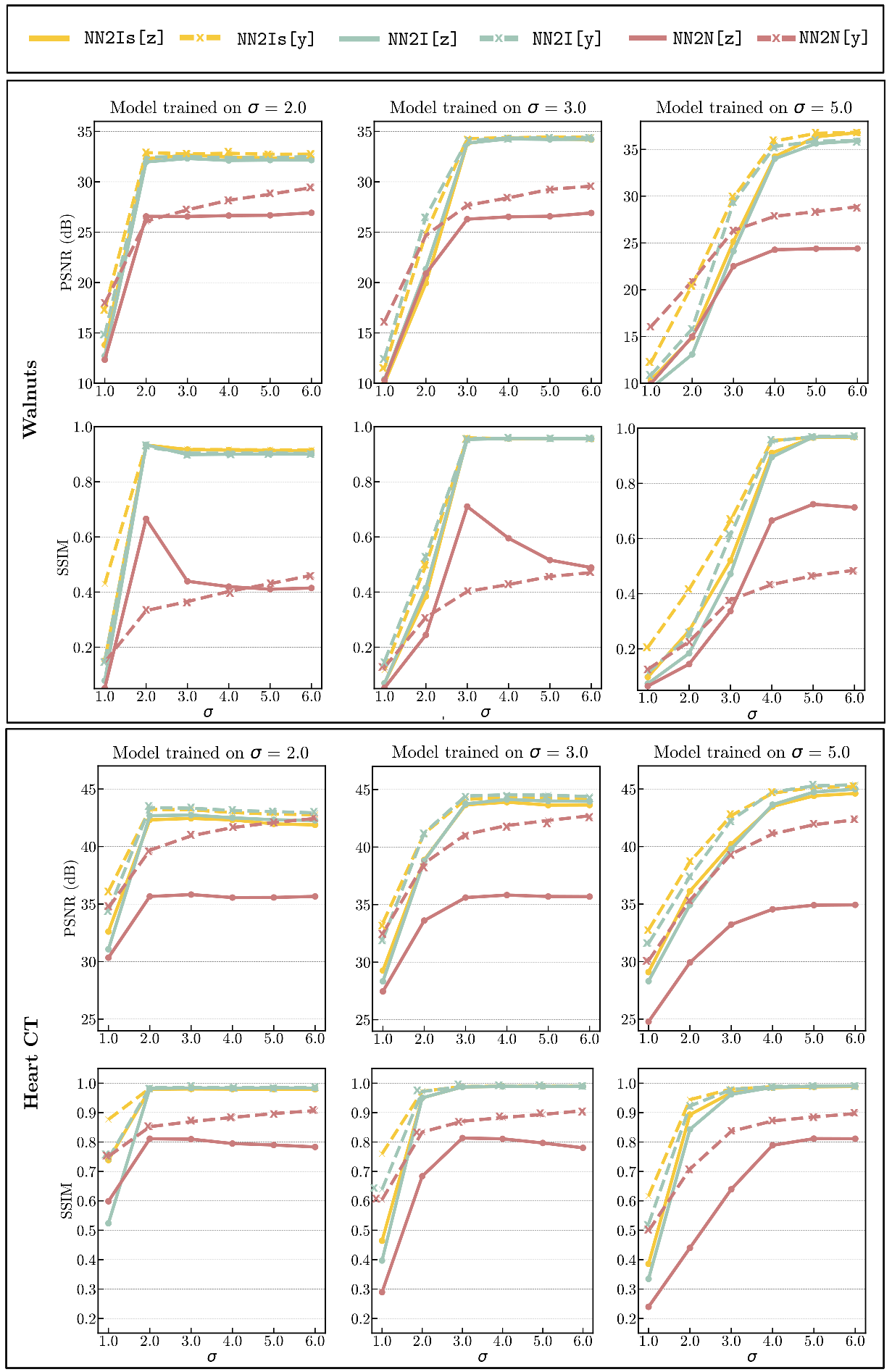}
\caption{\textbf{PSNR and SSIM for different correlation   parameters $\cl$ for the Walnut dataset (top) and the Heart CT dataset (bottom).} We display the mean metrics for 5 different correlation parameter $\cl$, while the network was trained on only one of them (2.0, 3.0, or 5.0). }
    \label{fig:sigmas}
\end{figure*}

\subsubsection{Robustness regarding the noise model}
Since the exact noise model may not always be known in practice, we assess the robustness of our trained models by evaluating their performance on data with varying correlation parameters. Specifically, we analyze how networks trained on data with a noise level corresponding to a kernel size
$\cl$ perform when applied to data with different values of $\cl$. The results, expressed in terms of PSNR and SSIM, are presented in Figure~\ref{fig:sigmas}. The graphs indicate the average performance metrics across all test images for six different kernel sizes $\cl$. The corresponding networks were trained on data degraded by correlated noise with kernel sizes of 2.0, 3.0, and 5.0, respectively.

In most cases, the highest performance is achieved when the model is tested on data with the same noise level it was trained on. 

A notable positive finding is that the Noisier2Inverse approaches maintain relatively stable performance even when applied to data with higher noise parameters $\cl$. In contrast, two-step methods exhibit a decline in performance under these conditions, making them less suitable for practical applications.

\clearpage

\subsubsection{Sparse data results}
Another key advantage of Noisier2Inverse is that the loss is computed directly in the data domain, enabling the removal of sparse data artifacts. To illustrate this, we performed experiments on sparse data with only 64 available projection angles. The results, presented in Figure~\ref{sparse}, clearly show that Noisier2Inverse outperforms the two-step approaches.

In the visualizations of the input images, dominant artifacts and correlated noise are evident. Both the numerical results and the boxplots confirm that the one-step approach achieves superior performance compared to the two-step method.

\begin{figure*}[htb!]
    \centering
    \includegraphics[width=\linewidth]{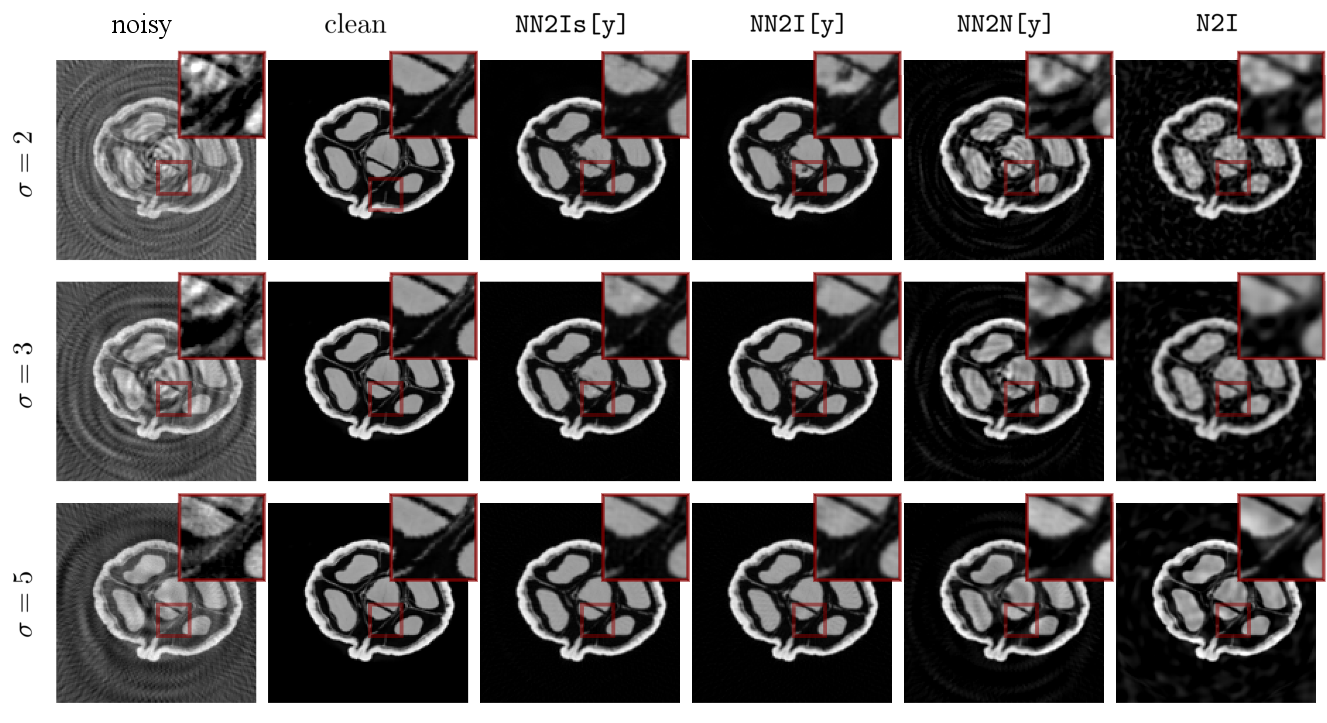}\\[1em]
    \includegraphics[width=\linewidth]{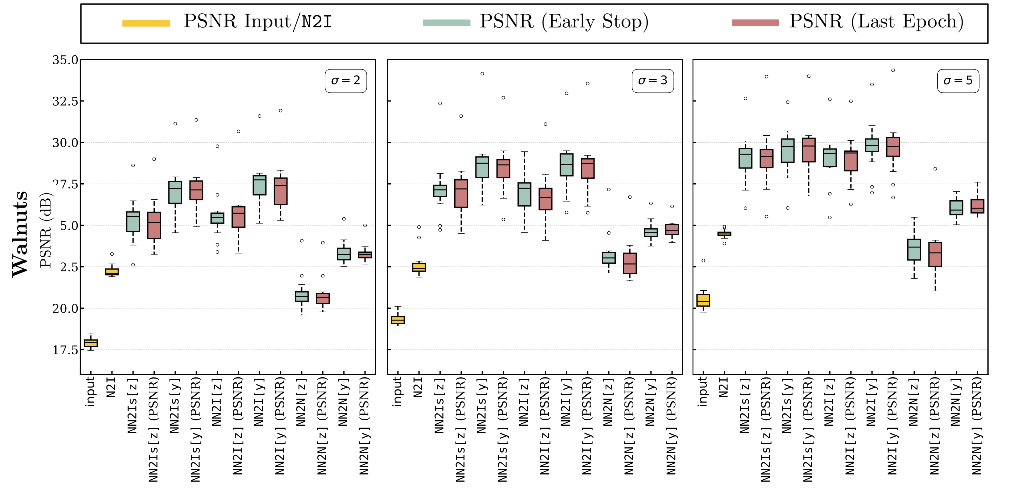}
    \caption{\textbf{Sparse data results.} Top: Results obtained on sparse projection data consisting of 64 angular directions using the $y-$prediction method. We can clearly observe the benefit of the one-step approaches compared to the two-step method. Bottom: PSNR and SSIM for the different methods obtained on sparse data examples (number of projection consisting of measurements from 64 different angular directions.}
    \label{sparse}
\end{figure*}

\section{Discussion}

In this work, we introduce a novel self-supervised method for the joint reconstruction and denoising of inverse problems, specifically designed to handle cases where the noise is correlated. This is useful in real-world imaging scenarios, where noise is often structured and non-Gaussian, such as in medical imaging modalities like MRI and CT, as well as in other applications where conventional methods struggle to perform effectively under these conditions. This assumption is consistent with many real-world scenarios. However, to the best of our knowledge, no existing method in the current literature explicitly addresses this challenge. We extend Noisier2Noise as a one-step approach \cite{moran2020noisier2noise} to inverse problems and further introduce Noisier2Inverse as an extrapolation-free alternative to Noisier2Noise that outperforms alternative reconstruction methods.

Consistent with observations in \cite{moran2020noisier2noise}, we find that Noisier2Inverse and Noisier2Noise using inference directly on noisy given data $y$ yield slightly better results than predictions made on the noisier data $z$. As a natural next step, it would be worthwhile to develop an adaptation of Noisier2Inverse, incorporating an additional parameter
to control the proportion of added noise. This adjustment could enable finer control over the denoising process. In particular, \cite{moran2020noisier2noise} reported improved performance when the intensity of the added noise was lower than that of the noise present in the original data, a finding that merits further investigation within this extended framework.

For future work, we plan to conduct more realistic experiments, where we simulate correlated noise, such as when a CT detector is defective or pieces of metal are in the human body. Moreover, applying Noisier2Inverse to other projection-based imaging tasks, such as Cryo-Electron Tomography, could be an interesting direction for future research. However, this falls outside the scope of this paper, which primarily focuses on establishing a solid methodological foundation.

\section{Conclusion}

In this work, we introduce Noisier2Inverse, a novel self-supervised image reconstruction approach for data corrupted by correlated noise. To achieve this, we have developed a new training strategy that performs well even with small datasets. Our results highlight that, despite assuming a known noise distribution and fixed noise parameters, Noisier2Inverse demonstrates robustness in various scenarios. Additionally, we show that Noisier2Inverse effectively addresses key limitations of existing state-of-the-art self-supervised methods, offering enhanced performance and stability. It is part of future work to further test and adapt the approach to more realistic scenarios, where the correlated noise is naturally apparent in the provided measurement data.

\section*{Acknowledgement}
G. Hwang was supported by the National Research Foundation of Korea(NRF) grant funded by the Korea government(MSIT) (RS-2024-00333393).

\bibliographystyle{IEEEtran}
\bibliography{Noiser2Inverse_Arxiv.bbl}

\clearpage

\begin{appendix}

\section{Supplementary Figures}
\label{sec:supplementary}

\begin{figure*}[htb!]
\centering
\includegraphics[width=\linewidth]{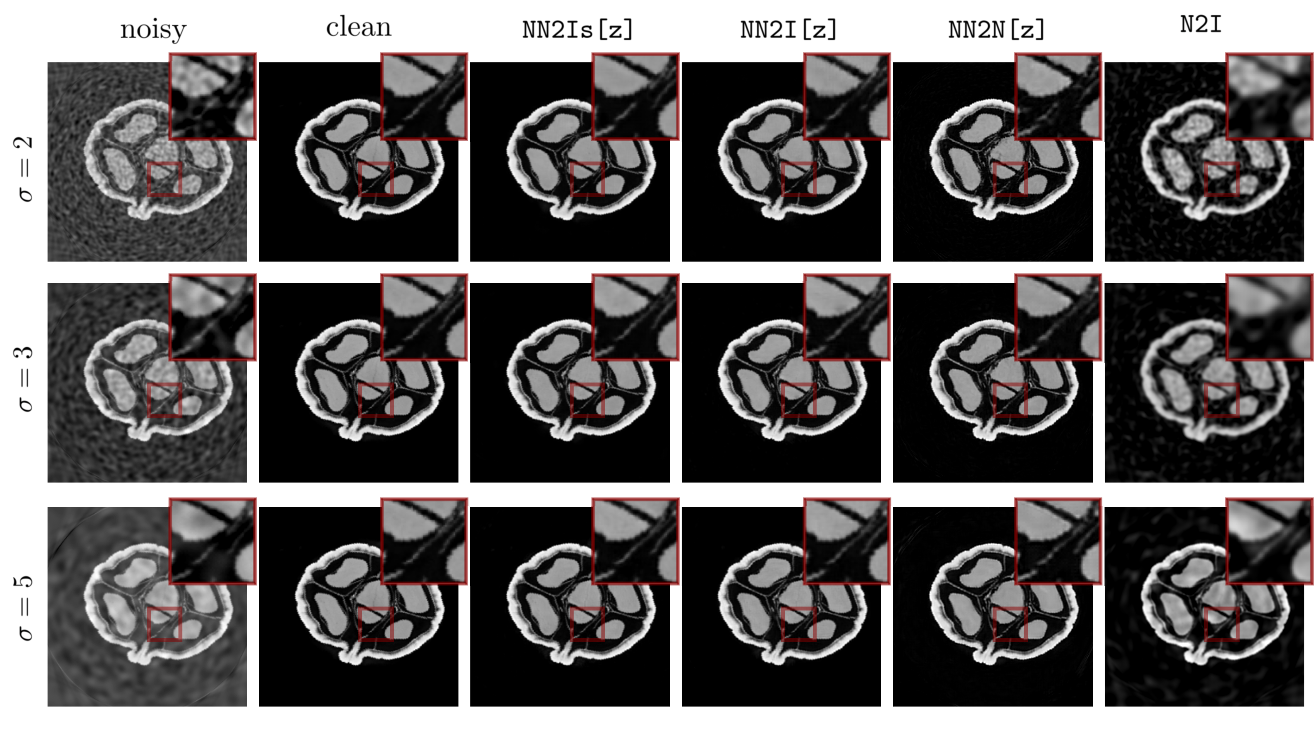}
\\[1em]
\includegraphics[width=\linewidth]{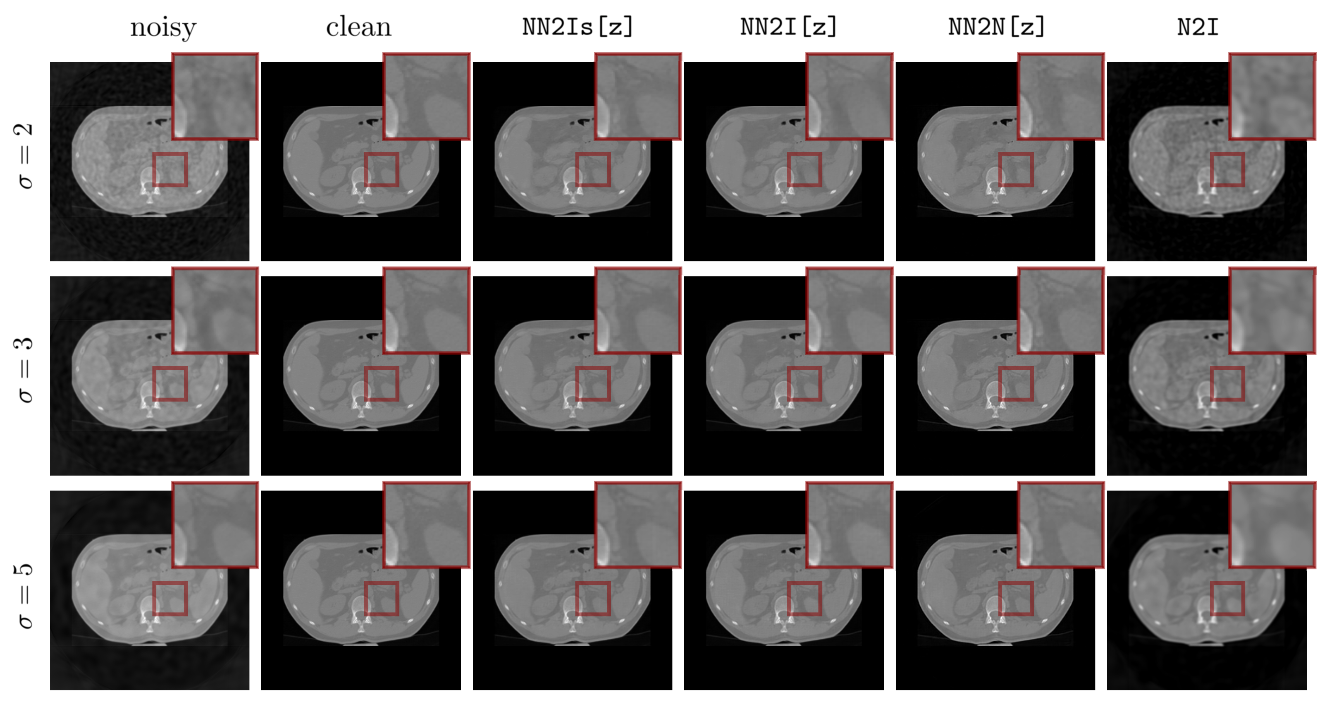}
    \caption{\textbf{Reconstruction results performing inference on the noisier data $z$} for the Walnut testing set (top) and the heart CT testing set (bottom).}
    \label{fig:z}
\end{figure*}

\begin{figure*}[htb!]
    \centering
\includegraphics[width=\linewidth]{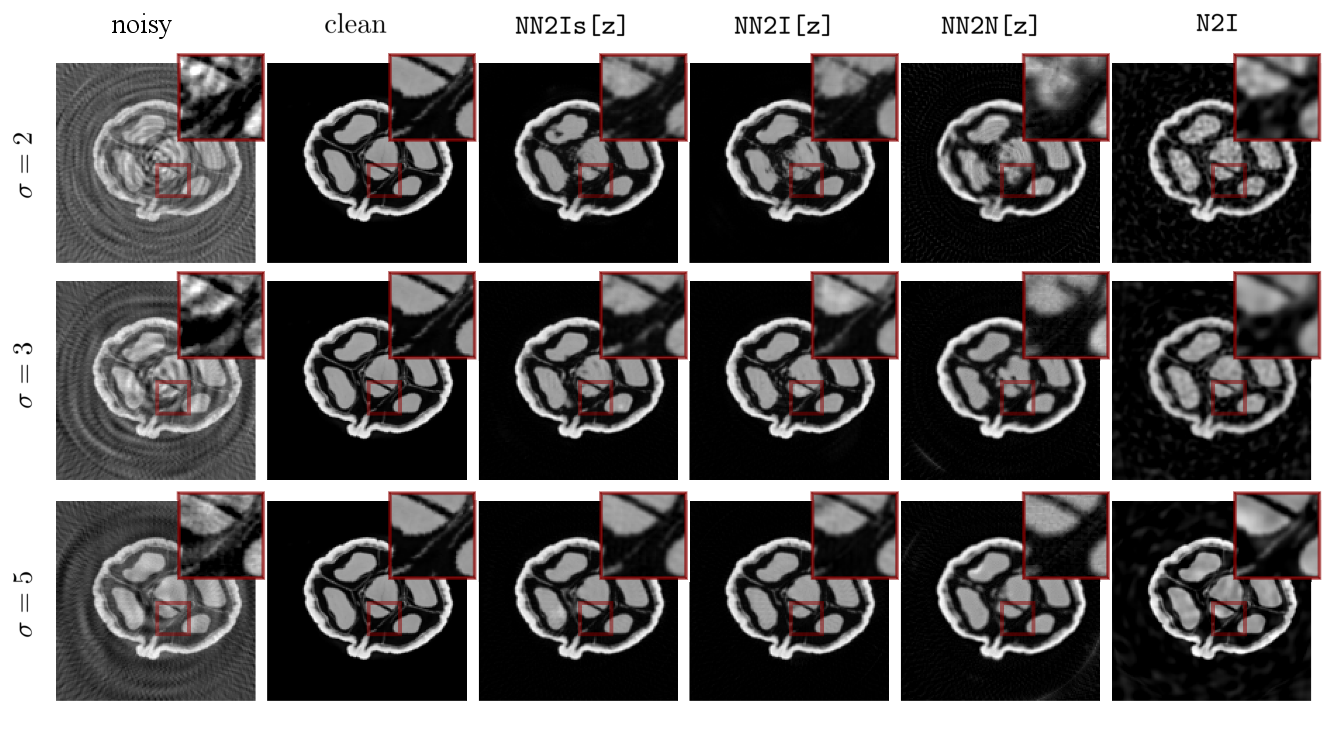}
    \caption{\textbf{Results obtained on sparse projection data consisting of 64 angular directions using the $z-$prediction method.} We can clearly observe the benefit of the one-step approaches compared to the twostep method.}
    \label{sparse_z}
\end{figure*}

\clearpage

\section{Supplementary Tables}

\begin{table*}[htb!]
    \setlength{\tabcolsep}{2.5pt} 

    \centering
    \caption{PSNR and SSIM Values for the different methods tested on noisy samples with noise parameter $\cl$.}
    \resizebox{\linewidth}{!}{
    \begin{tabular}{cccccccc}
        \toprule
        Walnut $\cl =2$ & \nys & \nzs & \ny & \nz & \nyt & \nzt & \nii\\
        \midrule
        \textbf{PSNR} & \textbf{32.86} (1.72) & 32.25 (1.93) & 32.59 (1.74) & 31.99 (1.83) & 26.24 (0.24) & 28.62 (1.17) & 22.27 (0.38) \\
        \textbf{SSIM} & 0.93 (0.06) & \textbf{0.94} (0.04) & 0.92 (1.74) & 0.0.93 (0.04) & 0.34 (0.05) & 0.68 (0.03) & 0.24 (0.04) \\
        \bottomrule
            \toprule
        Heart CT $\cl =2$ & \nys & \nzs & \ny & \nz & $y$ & $z$ & N2I \\
        \midrule
        \textbf{PSNR} & 43.25 (0.62) & 42.26 (0.57) & \textbf{43.26} (0.94) & 42.66  (0.85) & 39.54 (0.22) & 35.85 (1.39) & 32.08 (0.95) \\
        \textbf{SSIM} & \textbf{0.99 (0.002)} & \textbf{0.99 (0.002)} & 0.99 (1.08) & 0.98 (0.001) & 0.89 (0.005) & 0.83 (0.02) & 0.66 (0.01) \\
        \bottomrule
                \toprule
        Walnut $\cl =3$ & \nys & \nzs & \ny & \nz & $y$ & $z$ & N2I \\
        \midrule
        \textbf{PSNR} & 34.38 (1.80) & 33.90 (1.92) & \textbf{34.47} (1.63) & 33.78 (1.65) & 27.97 (0.26) & 30.03 (1.16) &22.71 (0.89) \\
        \textbf{SSIM} & \textbf{0.96 (0.03)}& 0.95 (0.03) & 0.96 (1.63) & 0.95 (0.03) & 0.40 (0.05) & 0.71 (0.03) & 0.27 (0.04) \\
        \bottomrule
            \toprule
        Heart CT $\cl =3$ & \nys & \nzs & \ny & \nz & $y$ & $z$ & N2I \\
        \midrule
        \textbf{PSNR} & 44.21 (1.04) & 43.71 (0.91) & \textbf{44.48} (1.08) & 43.87 (1.07) & 41.07 (0.20) & 36.15 (1.17) & 33.09 (1.01) \\
        \textbf{SSIM} & \textbf{0.99 (0.001)} & 0.98 (0.001) & 0.98 (1.07) & 0.980 (0.001) & 0.87 (0.005) & 0.82 (0.02) & 0.63 (0.01) \\
        \bottomrule
                \toprule
        Walnut $\cl =5$ & \nys & \nzs & \ny & \nz & $y$ & $z$ & N2I \\
        \midrule
        \textbf{PSNR} &  \textbf{36.75} (1.75) & 36.36 (1.58) & 36.36 (1.90) & 36.05 (1.94) & 29.83 (0.22)& 30.19 (0.96)  & 24.48 (0.27) \\
        \textbf{SSIM} &  \textbf{0.97 (0.02)}& \textbf{0.97 (0.02)} & 0.97 (1.91) & 0.96 (0.02) & 0.48 (0.05) & 0.69 (0.02) & 0.32 (0.05) \\
        \bottomrule
            \toprule
        Heart CT $\cl = 5$ & \nys & \nzs & \ny & \nz & $y$ & $z$ & N2I \\
        \midrule
        \textbf{PSNR} & 45.09 (0.61) & 44.39 (0.54) & \textbf{45.22} (1.07) & 44.73 (0.92) & 42.20 (0.30) & 36.30 (1.17) & 32.080 (0.01)\\
        \textbf{SSIM} & \textbf{0.99 (0.002)} & \textbf{0.99 (0.002)} & 0.99 (1.08) & 0.98 (0.001) & 0.89 (0.005) & 0.83 (0.02) & 0.66 (0.01) \\
        \bottomrule
    \end{tabular}
    }
    \label{tab:psnr_ssim}
\end{table*}

\begin{table*}[htb!]
    \centering
    \caption{PSNR and SSIM Values for the different methods tested on noisy samples with noise parameter $\cl$ on sparse data consisting of 64 projections.}
    \resizebox{\linewidth}{!}{
    \begin{tabular}{ccccccc}
        \toprule
        Walnut $\cl =2$ & \nys & \nzs & \ny & \nz & \nyt & \nzt \\
        \midrule
        \textbf{PSNR} & \textbf{27.12} (1.62) & 25.34 (1.43) & 27.36 (1.70) & 25.71 (1.74) & 23.33 (0.58) & 20.88 (1.07)  \\
        \textbf{SSIM} & 0.81 (0.05) & \textbf{0.82} (0.03) & 0.84 (1.70) &  0.83 (0.03) & 0.31 (0.03) & 0.25 (0.02)  \\
        \bottomrule
                \toprule
        Walnut $\cl =3$ & \nys & \nzs & \ny & \nz & $y$ & $z$ \\
        \midrule
        \textbf{PSNR} & 28.49 (1.69) & 27.15 (1.75) & \textbf{28.55} (1.89) & 26.72 (1.72) & 24.79 (0.53) & 22.97 (1.29) \\
        \textbf{SSIM} & \textbf{0.82 (0.03)}& 0.79 (0.03) & 0.82 (1.89) & 0.79 (0.03) & 0.38 (0.03) & 0.40 (0.03) \\
        \bottomrule
                \toprule
        Walnut $\cl =5$ & \nys & \nzs & \ny & \nz & $y$ & $z$ \\
        \midrule
        \textbf{PSNR} &  \textbf{29.67} (1.70) & 29.15 (1.92) & 29.71 (1.50) & 29.05 (1.50) & 26.16 (0.63) & 23.51 (1.72) \\
        \textbf{SSIM} &  \textbf{0.81 (0.03)}& \textbf{0.79 (0.03)} & 0.82 (1.82) & 0.80 (0.03) & 0.43 (0.05) & 0.40 (0.04)  \\
        \bottomrule
    \end{tabular}
    }
    \label{tab:sparse}
\end{table*}

\end{appendix}

\end{document}